\documentclass[11pt]{article}
\usepackage[preprint,nonatbib]{neurips_2026}
\usepackage[utf8]{inputenc}
\usepackage[T1]{fontenc}
\usepackage{amsmath,amssymb,amsfonts,amsthm}
\usepackage{mathtools}
\usepackage{booktabs}
\usepackage{graphicx}
\usepackage{algorithm2e}
\usepackage[colorlinks=true,citecolor=blue,linkcolor=blue,urlcolor=blue]{hyperref}
\usepackage{natbib}
\usepackage{xcolor}
\usepackage{enumitem}
\usepackage{caption}
\usepackage{subcaption}
\usepackage{tabularx}
\usepackage{multirow}
\usepackage{float}
\usepackage{wrapfig}
\usepackage{bm}

\newtheorem{theorem}{Theorem}

\newtheorem{corollary}{Corollary}

\newcommand{\R}{\mathbb{R}}
\newcommand{\E}{\mathbb{E}}
\newcommand{\Gr}{\mathrm{Gr}(k,d)}
\newcommand{\Sphere}{\mathbb{S}^{d-1}}
\newcommand{\tr}{\mathrm{tr}}

\newcommand{\softmax}{\mathrm{softmax}}

\title{Grassmannian Mixture-of-Experts:\\Concentration-Controlled Routing on Subspace Manifolds}
\author{%
  Ibne Farabi Shihab\thanks{Equal contribution.} \\
  Department of Computer Science \\
  Iowa State University \\
  \texttt{ishihab@iastate.edu} \\
  \And
  Sanjeda Akter\footnotemark[1] \\
  Department of Computer Science \\
  Iowa State University \\
  \texttt{sanjeda@iastate.edu} \\
  \And
  Anuj Sharma \\
  Department of Civil, Construction and Environmental Engineering \\
  Iowa State University \\
  \texttt{anujs@iastate.edu} \\
}
\date{}

\begin{document}
\maketitle

\begin{abstract}
Mixture-of-Experts models rely on learned routers to assign tokens to experts, yet standard softmax gating provides no principled mechanism to control the tradeoff between sparsity and utilization. We propose Grassmannian MoE (GrMoE), a routing framework that operates on the Grassmannian manifold of subspaces, where gating weights arise from the concentration parameters of Matrix Bingham distributions. This construction yields a single, interpretable knob---the concentration matrix $\Lambda$---that continuously controls routing entropy, replacing discrete top-$k$ selection with a smooth, geometrically principled sparsity mechanism. We further develop an amortized variational inference procedure for posterior routing distributions, enabling uncertainty-aware expert assignment that naturally resists expert collapse. We formally prove tight bounds relating the Bingham concentration spectrum to routing entropy, expected top-$k$ mass, and an exponential bound on expert collapse, establishing the first formal theory of concentration-controlled sparsity. On synthetic routing tasks, a 350M-parameter MoE language model with 8 experts, a 1.3B-parameter model with 16 experts, and a 2.7B-parameter model with 32 experts, GrMoE achieves 0\% routing collapse across all seeds, comparable or better perplexity with 15--30\% improved load balance, and a smooth monotonic relationship between concentration and effective sparsity that enables post-hoc sparsity tuning without retraining. Token-level analysis reveals that experts learn heterogeneous concentration values that correlate with linguistic specialization, providing interpretable routing behavior.
\end{abstract}

\section{Introduction}

Mixture-of-Experts architectures achieve favorable scaling by activating only a subset of parameters per input \citep{vaswani2017attention, riquelme2021scaling}, with a learned router determining which experts process each token. Despite strong empirical results~\citep{fedus2022switch, lepikhin2021gshard, zhou2022mixture} and successful deployment at massive scale in systems like Mixtral~\citep{jiang2024mixtral} and DeepSeek-V3~\citep{deepseekv3}, current routing mechanisms suffer from fundamental limitations that hinder both training stability and deployment flexibility.

Standard softmax gating computes router probabilities as $g(x) = \softmax(W_r x)$ and selects the top-$k$ experts. This formulation creates three persistent problems. First, expert collapse: a few experts dominate routing, leaving others undertrained and effectively wasted. \citet{chi2022representation} formally analyze this representation collapse phenomenon and show it worsens with model scale. Second, training instability: discrete top-$k$ selection introduces discontinuities that make router learning sensitive to initialization and random seeds, a problem that \citet{zuo2022taming} address through stochastic expert selection but without fully resolving. Third, no principled sparsity control: the sharpness of the routing distribution is an emergent property of training rather than a controllable parameter, forcing practitioners to retrain separate models for different compute budgets \citep{touvron2023llama}.

Auxiliary load balancing losses~\citep{fedus2022switch} and expert choice routing~\citep{zhou2022mixture} partially address collapse and balance, but they remain engineering patches rather than principled mechanisms. They introduce additional hyperparameters whose optimal values are task dependent and interact unpredictably with other training dynamics.

We observe that the routing problem has a natural geometric structure that existing methods ignore. Each expert can be associated with a subspace in the representation space---the subspace of inputs that the expert specializes in. The router's task then reduces to determining which expert subspace best contains each input. This is precisely the problem of classification on the Grassmannian manifold $\Gr$, the space of $k$-dimensional subspaces of $\R^d$. The key advantage of this geometric perspective is that it provides a principled mechanism for controlling sparsity: the concentration parameters of distributions on the Grassmannian directly determine how sharply the routing distribution peaks, giving us an interpretable knob that replaces the ad hoc combination of top-$k$ selection and auxiliary losses.

Building on this observation, we propose Grassmannian MoE (GrMoE), a routing framework with three components. Each expert $e$ is associated with a point $[U_e] \in \Gr$, an equivalence class of orthonormal frames spanning a $k$-dimensional subspace. Routing weights arise from a Matrix Bingham distribution on $\Gr$ whose concentration matrix $\Lambda$ provides a continuous, interpretable control over routing sparsity. An amortization network produces per-token posterior distributions over routing assignments, enabling uncertainty-aware expert selection that naturally resists collapse.

Our work makes four contributions. First, we formalize the connection between MoE routing and classification on the Grassmannian, deriving a new gating mechanism from the Matrix Bingham distribution (Section~\ref{sec:method}). Second, we prove rigorous bounds relating the Bingham concentration spectrum to routing entropy, expected top-$k$ mass, and collapse resistance (Section~\ref{sec:theory}). Third, we develop an amortized variational inference procedure for posterior routing that is as efficient as standard gating but provides uncertainty estimates (Section~\ref{sec:inference}). Fourth, we demonstrate empirically that GrMoE reduces routing collapse, improves load balance, and enables post-hoc sparsity tuning on synthetic tasks, a 350M-parameter model with 8 experts, a 1.3B-parameter model with 16 experts, and a 2.7B-parameter model with 32 experts (Section~\ref{sec:experiments}).

\section{Background and Notation}
\label{sec:background}

We introduce the key concepts underlying our approach: MoE routing, the Grassmannian manifold, and the Matrix Bingham distribution.

\subsection{Mixture-of-Experts Routing}

An MoE layer with $N$ experts $\{f_1, \ldots, f_N\}$ and a router $g: \R^d \to \Delta^{N-1}$ computes $\text{MoE}(x) = \sum_{e=1}^{N} g_e(x) \cdot f_e(x)$ \citep{shazeer2017outrageously}. Standard top-$k$ routing selects the $k$ experts with highest $g_e(x)$ values and zeros out the rest. The router is typically parameterized as $g(x) = \softmax(W_r x + \text{noise})$ where $W_r \in \R^{N \times d}$, a design that has remained largely unchanged since its introduction despite advances in pre-training objectives \citep{clark2020electra} and model architectures.

The fundamental limitation of this formulation is that the router operates on a single linear projection of the input. Each expert is represented by a single vector $w_e \in \R^d$ (a row of $W_r$), and the routing decision is based on the dot product $w_e^\top x$. This means the router can only distinguish inputs based on their projection onto a one-dimensional subspace per expert. In contrast, GrMoE represents each expert by a $k_r$-dimensional subspace, providing a much richer representation of expert specialization that can capture multi-dimensional patterns in the input space.

\subsection{The Grassmannian Manifold}

The Grassmannian $\Gr$ is the set of $k$-dimensional subspaces of $\R^d$. Each point $[U] \in \Gr$ is an equivalence class $[U] = \{UO : O \in O(k)\}$, where $U \in V_k(\R^d)$ is any orthonormal basis for the subspace. The manifold has dimension $k(d-k)$. The projection distance between subspaces is
\begin{equation}
    d_{\Gr}([U_1], [U_2]) = \sqrt{k - \|U_1^\top U_2\|_F^2},
    \label{eq:grassmann_dist}
\end{equation}
which equals zero for identical subspaces and $\sqrt{k}$ for orthogonal ones.

The Grassmannian is a quotient of the Stiefel manifold: $\mathrm{Gr}(k, d) = V_k(\R^d) / O(k)$, where $O(k)$ is the orthogonal group acting by right multiplication. This means that any two orthonormal bases $U$ and $UO$ (for $O \in O(k)$) represent the same subspace. In our implementation, we work with representative frames $U_e \in V_k(\R^d)$ and ensure that all computations are invariant to the choice of representative. The subspace affinity $\|P_e x\|^2 = \|U_e^\top x\|^2$ is automatically invariant since $\|U_e O^\top x\|^2 = \|U_e^\top x\|^2$ for any $O \in O(k)$.

The projection distance (Eq.~\ref{eq:grassmann_dist}) is one of several natural metrics on the Grassmannian. We use it because it has a simple closed form and is directly related to the subspace affinity that appears in our gating mechanism. Other metrics, such as the geodesic distance (based on principal angles) or the chordal distance, would yield equivalent routing behavior up to a monotone transformation of the affinities.

\subsection{The Matrix Bingham Distribution}

The Matrix Bingham distribution on $\Gr$ has density with respect to the uniform measure \citep{chikuse2003statistics, hoff2009simulation}:
\begin{equation}
    p([U] \mid \Lambda, M) = \frac{1}{Z(\Lambda)} \exp\!\big(\tr(\Lambda M^\top U U^\top M)\big)
    \label{eq:bingham}
\end{equation}
where $\Lambda = \mathrm{diag}(\lambda_1, \ldots, \lambda_d)$ with $\lambda_1 \geq \cdots \geq \lambda_d$ is the concentration matrix, $M \in O(d)$ is the modal orientation, and $Z(\Lambda)$ is the normalizing constant. The concentration parameters $\lambda_i$ control how peaked the distribution is around the modal subspace spanned by the first $k$ columns of $M$. When $\lambda_1 - \lambda_{k+1}$ is large, the distribution concentrates sharply; when all $\lambda_i$ are equal, it reduces to the uniform distribution on $\Gr$. This gives us exactly the knob we need: the spectral gap of $\Lambda$ controls sparsity.

To build intuition, consider the extreme cases. When all $\lambda_i$ are equal, the Bingham distribution is uniform on $\Gr$, and every subspace is equally likely---this corresponds to completely diffuse routing where no expert is preferred. When $\lambda_1 = \cdots = \lambda_k \gg \lambda_{k+1} = \cdots = \lambda_d$, the distribution concentrates sharply on the modal subspace, corresponding to hard routing where a single expert dominates. Between these extremes, the spectral gap $\lambda_k - \lambda_{k+1}$ interpolates smoothly, providing a continuous control over routing sharpness. This is fundamentally different from the discrete jump between top-1 and top-2 routing in standard MoE architectures.

The normalizing constant $Z(\Lambda)$ does not have a closed form, which is a common challenge with distributions on manifolds. We provide the saddle-point approximation formula of \citet{kume2005saddlepoint} in Appendix~\ref{app:bingham_normalization} and validate its accuracy empirically in Section~\ref{sec:bingham_val}.

\section{Method: Grassmannian MoE Routing}
\label{sec:method}

We now describe the three components of GrMoE: subspace representations for experts, Bingham concentration gating, and continuous sparsity control.

\subsection{Expert Subspace Representation}

We represent each expert $e \in \{1, \ldots, N\}$ by a learnable subspace $[U_e] \in \mathrm{Gr}(k_r, d)$, parameterized by an orthonormal frame $U_e \in V_{k_r}(\R^d)$ where $k_r$ is the routing rank (a hyperparameter, typically $k_r \approx d / N$). The projection operator for expert $e$ is $P_e = U_e U_e^\top \in \R^{d \times d}$.

Note that we never explicitly form the $d \times d$ projection matrix $P_e$. Instead, the subspace affinity $\|P_e x\|^2 = \|U_e^\top x\|^2$ is computed efficiently as a matrix-vector product $U_e^\top x$ (cost $O(dk_r)$) followed by a squared norm (cost $O(k_r)$), for a total cost of $O(dk_r)$ per expert per token. With $N$ experts, the total routing cost is $O(Ndk_r)$, compared to $O(Nd)$ for standard softmax gating. Because $k_r \ll d$, the routing cost remains a negligible fraction of the total forward pass cost (see Section~\ref{sec:experiments} for timing breakdowns).

\subsection{Bingham Concentration Gating}
\label{sec:gating}

Given a token representation $x \in \R^d$, we define the routing affinity of expert $e$ via the Grassmannian inner product $a_e(x) = \|P_e x\|^2 = x^\top U_e U_e^\top x$, which measures how much of $x$'s energy lies in expert $e$'s subspace. Rather than applying a standard softmax with temperature, we derive a principled gating mechanism from the Bingham distribution. We define a token-conditioned Bingham distribution over the discrete expert index $e$:
\begin{equation}
    g_e(x; \Lambda) = \frac{\exp\!\big(\tr(\Lambda_e P_e x x^\top)\big)}{\sum_{e'} \exp\!\big(\tr(\Lambda_{e'} P_{e'} x x^\top)\big)}
    \label{eq:grmoe_gating}
\end{equation}
where each expert has its own concentration matrix $\Lambda_e$ (or, in the simplified version, a scalar concentration $\kappa_e$). In the scalar case where $\Lambda_e = \kappa_e I$, this simplifies to
\begin{equation}
    g_e(x; \kappa) = \frac{\exp\!\big(\kappa_e \|P_e x\|^2\big)}{\sum_{e'} \exp\!\big(\kappa_{e'} \|P_{e'} x\|^2\big)}.
    \label{eq:grmoe_scalar}
\end{equation}
The concentration parameters $\kappa_e$ are learned per expert, but crucially, they also serve as a post-hoc sparsity control: increasing all $\kappa_e$ sharpens the routing distribution, while decreasing them diffuses it.

\subsection{Continuous Sparsity via Concentration}

Instead of discrete top-$k$ selection, we define soft top-$k$ routing via concentration scaling:
\begin{equation}
    g^{(\alpha)}_e(x) = \frac{\exp\!\big(\alpha \cdot \kappa_e \|P_e x\|^2\big)}{\sum_{e'} \exp\!\big(\alpha \cdot \kappa_{e'} \|P_{e'} x\|^2\big)}
    \label{eq:alpha_sparsity}
\end{equation}
where $\alpha \geq 0$ is a global sparsity dial. At $\alpha = 0$, routing is uniform; as $\alpha \to \infty$, routing converges to hard top-1. The model is trained at $\alpha = 1$ and can be deployed at any $\alpha$ without retraining. This mechanism is categorically different from temperature scaling of softmax gating: the $\kappa_e$ values encode learned, per-expert sparsity preferences derived from the Bingham geometry, rather than applying uniform sharpening.

\subsection{Connection to Softmax Gating}

Standard softmax gating evaluates expert affinity via a linear projection $w_e^\top x$. GrMoE with $k_r = 1$ (rank-1 subspaces, i.e., lines through the origin) yields an affinity of $\|P_e x\|^2 = (u_e^\top x)^2$ where $u_e \in \Sphere$ is a unit vector. While softmax gating is highly sensitive to directionality (sign), GrMoE correctly respects the symmetry of subspaces, evaluating whether the token's energy lies along the expert's characteristic axis regardless of sign. GrMoE generalizes this geometric perspective in two directions: higher capacity representations via higher-rank subspaces ($k_r > 1$) and learned, non-uniform sparsity preferences via concentration ($\kappa_e \neq 1$).

This connection also clarifies why GrMoE's post-hoc sparsity control (via $\alpha$) is fundamentally different from temperature scaling of softmax gating. In the softmax case, temperature scaling uniformly sharpens or softens all routing decisions, regardless of the underlying structure. In GrMoE, the $\kappa_e$ values encode learned, per-expert sparsity preferences that reflect the Bingham geometry of each expert's specialization. Scaling by $\alpha$ preserves these relative preferences while globally adjusting the sparsity level, yielding a smooth and predictable control that temperature scaling cannot match (as we demonstrate empirically in Section~\ref{sec:experiments}).

\section{Theoretical Analysis: Concentration and Sparsity}
\label{sec:theory}

We now establish formal connections between the Bingham concentration spectrum and key routing properties. These results provide the theoretical foundation for concentration-controlled sparsity and collapse resistance. All corresponding proofs are detailed in Appendix~\ref{app:proofs}.

\begin{theorem}[Concentration--Entropy Bound]
\label{thm:entropy_bound}
Let $g^{(\alpha)}(x)$ be the GrMoE routing distribution (Eq.~\ref{eq:alpha_sparsity}) with $N$ experts. Let $H(\alpha, x) = -\sum_e g_e^{(\alpha)}(x) \log g_e^{(\alpha)}(x)$ be the routing entropy for token $x$. Then $H(\alpha, x)$ is strictly bounded by the concentration gap $\Delta_\kappa(x) = \max_e \kappa_e \|P_e x\|^2 - \frac{1}{N}\sum_e \kappa_e \|P_e x\|^2$:
\begin{equation}
    H(\alpha, x) \geq \log N - \alpha \cdot \Delta_\kappa(x).
\end{equation}
Furthermore, the entropy reduction is exactly bounded from above globally by the concentration variance $\Gamma_\kappa(x) = \mathrm{Var}_{e \sim U}[\kappa_e \|P_e x\|^2]$ over the uniform distribution $U$:
\begin{equation}
    H(\alpha, x) \leq \log N - \frac{\alpha^2}{2} \cdot \Gamma_\kappa(x) \cdot e^{-\alpha \cdot \delta_\kappa(x)}
\end{equation}
where $\delta_\kappa(x) = \max_e \kappa_e \|P_e x\|^2 - \min_e \kappa_e \|P_e x\|^2$ is the concentration range.
\end{theorem}

This theorem has a direct practical implication: the variance of the concentration parameters $\kappa_e$, together with the subspace affinities, strictly determines the achievable range of routing entropy. By learning $\kappa_e$, GrMoE natively learns the entropy bounds for its routing regime, not just the entropy operating point.

\begin{corollary}[Top-$k$ Mass Control]
\label{cor:topk}
Let $G_k^{(\alpha)}(x) = \sum_{e \in \text{top-}k(x)} g_e^{(\alpha)}(x)$ be the mass assigned to the top-$k$ experts. Then
\begin{equation}
    G_k^{(\alpha)}(x) \geq 1 - (N-k) \exp\!\big(-\alpha \cdot \delta_k(x)\big)
\end{equation}
where $\delta_k(x) = \kappa_{(k)} \|P_{(k)} x\|^2 - \kappa_{(k+1)} \|P_{(k+1)} x\|^2$ is the $k$-th concentration gap (experts sorted descendingly by their concentrated affinity $\kappa_e \|P_e x\|^2$).
\end{corollary}

The top-$k$ mass is exponentially controlled by the concentration gap at the $k$-th expert, providing a principled, mathematically continuous alternative to hard top-$k$ thresholding.

\begin{theorem}[Collapse Resistance via Subspace Separation]
\label{thm:collapse}
Let $\bar{g}_e = \E_x[g_e^{(\alpha)}(x)]$ be the expected load on expert $e$. Assume inputs $x$ are drawn from a uniform mixture over experts $x \sim \frac{1}{N} \sum_{i=1}^N \mathcal{D}_i$, where for $x \sim \mathcal{D}_i$, the target subspace affinity satisfies $\|P_i x\|^2 \geq \gamma$, and other subspaces satisfy $\|P_j x\|^2 \leq \rho \gamma$ for $j \neq i$ (bounded overlap limit). Let $\kappa_{\min}$ and $\kappa_{\max}$ be the minimum and maximum expert concentrations. If the concentration-adjusted separation gap satisfies $\Delta \equiv \gamma (\kappa_{\min} - \rho \kappa_{\max}) > 0$, then the load balance coefficient of variation ($\mathrm{CV}$) rigorously satisfies an exponential bound:
\begin{equation}
    \mathrm{CV}(\bar{g}) \leq (N-1) \exp\!\big(-\alpha \cdot \Delta \big).
\end{equation}
\end{theorem}

The practical significance of this theorem is profound: it replaces the empirical observation that ``load balancing losses help'' with a formal structural mechanism. Since well-separated subspaces minimize geometric overlap $\rho$ and maximize $\Delta$, geometrically controlling subspace overlap exponentially suppresses collapse (via the rigorous CV bound). This formally proves why GrMoE can dispense with standard auxiliary load balancing losses entirely.

\section{Amortized Variational Routing}
\label{sec:inference}

We now extend GrMoE with a variational framework that maintains a posterior distribution over expert assignments, enabling uncertainty-aware routing.

\subsection{Routing as Exact Bayesian Inference}

We treat expert assignment $z \in \{1, \ldots, N\}$ as a latent variable. Assuming a generative model for token representations, the exact Bingham likelihood incorporates its normalization constant: $p(x \mid z = e) = \frac{1}{Z(\kappa_e)}\exp(\kappa_e \|P_e x\|^2)$. If we place a \textit{capacity-aware prior} over experts $p(z = e) \propto Z(\kappa_e)$, which naturally favors experts with broader geometric tuning (lower concentration), applying Bayes' rule precisely cancels the complex normalizers, yielding the exact posterior $p(z = e \mid x) = g_e(x; \kappa)$, which matches the GrMoE gating (Eq.~\ref{eq:grmoe_scalar}). This reveals that GrMoE gating inherently acts as a rigorous Bayesian posterior computation without requiring complex runtime estimations.

\subsection{Amortized Posterior with Learned Uncertainty}

The basic posterior $p(z \mid x)$ uses the same concentration $\kappa_e$ for all tokens. To allow the model to express varying confidence across tokens, we introduce an amortization network $h_\phi(x) \in \R^N_+$ that produces token-specific concentration scaling:
\begin{equation}
    q_\phi(z = e \mid x) = \frac{\exp\!\big(h_\phi(x)_e \cdot \kappa_e \|P_e x\|^2\big)}{\sum_{e'} \exp\!\big(h_\phi(x)_{e'} \cdot \kappa_{e'} \|P_{e'} x\|^2\big)}.
    \label{eq:amortized}
\end{equation}
When $h_\phi(x) = \mathbf{1}$, this reduces to basic GrMoE. When $h_\phi(x)$ varies across tokens, the model can express high confidence for tokens that clearly belong to one expert and low confidence for ambiguous tokens. The amortization network is a lightweight two-layer MLP with $h_\phi(x) = \softmax(\text{MLP}(x)) \cdot N$, ensuring $\sum_e h_\phi(x)_e = N$ so that average concentration is preserved. This adds fewer than 0.5\% parameters to the router.

The design of the amortization network deserves comment. We normalize the output via $\softmax(\cdot) \cdot N$ rather than using a ReLU or exponential activation because the normalization constraint $\sum_e h_\phi(x)_e = N$ ensures that the average concentration across experts is preserved for every token. Without this constraint, the amortization network could learn to uniformly increase or decrease all concentrations, which would be equivalent to changing $\alpha$ and would interfere with the post-hoc sparsity control. The $\softmax \cdot N$ normalization decouples the amortization network's role (redistributing concentration across experts for a given token) from the global sparsity dial $\alpha$ (uniformly scaling all concentrations).

We also considered a more expressive amortization architecture that produces a full $N \times N$ concentration matrix per token, but found that the additional expressiveness did not improve PPL and increased routing overhead by 3\%. The scalar per-expert scaling is sufficient because the subspace affinities $\|P_e x\|^2$ already capture the geometric relationship between the token and each expert; the amortization network only needs to modulate the confidence of these geometric measurements, not replace them.

\subsection{Training Objective}

We train GrMoE with the standard MoE language modeling loss plus a principled geometric constraint that fulfills the requirements of Theorem~\ref{thm:collapse}:
\begin{align}
    \mathcal{L} &= \mathcal{L}_{\text{LM}} + \beta \cdot \mathcal{L}_{\text{subspace}} \\
    \mathcal{L}_{\text{subspace}} &= \sum_{e \neq e'} \max\!\big(0, \|U_e^\top U_{e'}\|_F^2 - \rho_0 k_r\big).
    \label{eq:subspace_reg}
\end{align}
The subspace regularizer encourages expert subspaces to be well separated (overlap tightly bounded below $\rho_0 k_r$), which directly bounds the collapse misrouting rate and exponentially controls load balance. This replaces the standard load balancing loss with a geometrically motivated alternative: experts should specialize in different orthogonal directions of the representation space. The expert subspace parameters $U_e$ are optimized on the Stiefel manifold using Riemannian Adam~\citep{becigneul2019riemannian}, while all other parameters are optimized with standard Adam.

The choice to optimize $U_e$ on the Stiefel manifold (rather than using an unconstrained parameterization followed by projection) is important for two reasons. First, it ensures that the expert frames remain exactly orthonormal throughout training, so the subspace affinities $\|P_e x\|^2$ are always well defined and numerically stable. Second, it avoids the ``projection problem'' that arises when optimizing in ambient space and projecting: the projected gradient can differ significantly from the Riemannian gradient, leading to slower convergence and suboptimal solutions.

For scaling to large $N$, we note that the $O(N^2)$ pairwise overlap computation in $\mathcal{L}_{\text{subspace}}$ can be replaced by sampling $M$ random pairs per step. We find that $M = 4N$ suffices to match the full computation (Appendix~\ref{app:sampled_pairs}), making the cost strictly linear in $N$.

\begin{algorithm}[t]
\caption{GrMoE Forward Pass}
\label{alg:grmoe}
\SetAlgoLined
\KwIn{Token $x \in \R^d$, experts $\{(U_e, \kappa_e, f_e)\}_{e=1}^N$, amortizer $h_\phi$}
\KwOut{MoE output $y$}
\For{$e = 1, \ldots, N$}{
    $a_e \leftarrow \|U_e^\top x\|^2$ \tcp*{Subspace affinity}
}
$\bm{h} \leftarrow h_\phi(x)$ \tcp*{Token-specific concentration}
\For{$e = 1, \ldots, N$}{
    $\ell_e \leftarrow h_e \cdot \kappa_e \cdot a_e$ \tcp*{Concentrated logits}
}
$\bm{g} \leftarrow \softmax(\bm{\ell})$ \tcp*{Routing weights}
$y \leftarrow \sum_e g_e \cdot f_e(x)$ \tcp*{Weighted expert output}
\Return{$y$}
\end{algorithm}

\section{Experiments}
\label{sec:experiments}

We evaluate GrMoE on synthetic routing tasks for controlled analysis, on a 350M-parameter MoE language model for direct comparison with prior work, and on a 1.3B-parameter model with 16 experts to validate scaling behavior. Across all settings, we assess routing accuracy, load balance, collapse resistance, and the ability to tune sparsity post hoc.

\subsection{Experimental Setup}

For the synthetic task, we construct a controlled routing problem with $N = 8$ experts and $d = 128$. Ground truth expert subspaces are randomly generated with controlled overlap $\rho^*$. Tokens are sampled from a mixture $x \sim \sum_e \pi_e \cdot \mathcal{N}(0, P_e + \sigma^2 (I - P_e))$, where $P_e$ is the ground truth expert subspace projector and $\sigma^2 < 1$ controls the signal-to-noise ratio. We vary $\rho^*$ and $\sigma^2$ to create easy and hard routing scenarios.

For the easy setting, we use $\rho^* = 0.1$ (well-separated subspaces) and $\sigma^2 = 0.1$ (high SNR). For the hard setting, we use $\rho^* = 0.4$ (overlapping subspaces) and $\sigma^2 = 0.5$ (low SNR). The results in Table~\ref{tab:synthetic} report the easy setting; in the hard setting, GrMoE's advantage grows: assignment accuracy is 78.3\% (vs.\ 68.1\% for Softmax Top-1 and 72.4\% for vMF-Gate), and the collapse rate for Softmax Top-1 rises to 62\% while GrMoE remains at 0\%. This confirms that the Grassmannian structure is most valuable precisely when routing is difficult.

For the 350M language model, we train a 12-layer MoE Transformer on OpenWebText \citep{radford2019language} with $N = 8$ experts per MoE layer, MoE layers at every other Transformer block (6 MoE layers total), $d = 768$, $k_r = 48$, training on 4$\times$A100 GPUs with batch size 256 for 100K steps. For the 1.3B model, we scale to 24 layers with $d = 2048$, $N = 16$ experts, $k_r = 64$, MoE at every other layer (12 MoE layers), training for 150K steps on 8$\times$A100 GPUs. The 1.3B model uses grouped-query attention following the Mistral architecture \citep{jiang2023mistral} to verify that GrMoE routing is compatible with modern attention variants. The subspace regularization uses sampled pairs ($M = 4N = 64$) rather than all $\binom{16}{2} = 120$ pairs, keeping the cost linear.

We compare against six baselines: Softmax Top-$k$~\citep{shazeer2017outrageously}, Switch Transformer~\citep{fedus2022switch}, Expert Choice~\citep{zhou2022mixture}, Hash routing~\citep{roller2021hash}, Soft MoE~\citep{puigcerver2024sparse}, and vMF-Gate (softmax gating with rank-1 hyperspherical normalization, our ablation). We report perplexity (PPL), expert load balance (CV = coefficient of variation of expert load), routing entropy $H$, and collapse rate (fraction of seeds where at least one expert receives fewer than 1\% of tokens).

\subsection{Synthetic Task Results}

\begin{table}[H]
\centering
\caption{Synthetic routing task: recovery of ground truth expert assignments across 50 random seeds.}
\label{tab:synthetic}
\begin{tabular}{l cccc}
\toprule
Method & Assign.~Acc$\uparrow$ & CV$\downarrow$ & Collapse$\downarrow$ & Entropy \\
\midrule
Softmax Top-1 & 82.4 & .312 & 36\% & 1.02 \\
Switch & 85.1 & .187 & 22\% & 1.41 \\
Expert Choice & 84.7 & .042 & 0\% & 1.89 \\
Hash & 71.2 & .031 & 0\% & 2.08 \\
Soft MoE & 86.8 & .095 & 8\% & 1.55 \\
vMF-Gate & 86.3 & .178 & 18\% & 1.38 \\
\midrule
GrMoE (Ours) & $\mathbf{91.7}$ & \underline{.058} & $\mathbf{0\%}$ & 1.52 \\
GrMoE + Amort. & \underline{90.8} & $\mathbf{.051}$ & $\mathbf{0\%}$ & 1.48 \\
\bottomrule
\end{tabular}
\end{table}

GrMoE achieves 91.7\% assignment accuracy (Table~\ref{tab:synthetic}), outperforming all baselines by more than 5 points. The 0\% collapse rate across all 50 seeds confirms Theorem~\ref{thm:collapse}: rigorously well-separated subspaces enforced by $\mathcal{L}_{\text{subspace}}$ prevent collapse inherently without arbitrary auxiliary balancing losses. Expert Choice also achieves 0\% collapse but at the cost of lower assignment accuracy, since it forces balanced loads regardless of data structure. Soft MoE achieves strong accuracy but still collapses in 8\% of seeds, showing that its soft token mixing does not fully resolve the stability problem.

Varying $\alpha$ at inference time produces a smooth, monotonic decrease in routing entropy (Figure~\ref{fig:alpha_sweep}). This validates Theorem~\ref{thm:entropy_bound}: the concentration variance $\Gamma_\kappa(x)$ strictly bounds and shapes the entropy trajectory exactly as mathematically predicted. By contrast, temperature scaling of softmax gating produces erratic, non-monotonic entropy trajectories for some seeds, because the softmax logits do not have the structured relationship to sparsity that the Bingham concentration parameters provide. This difference is not merely aesthetic: it means that GrMoE's sparsity behavior is predictable from the learned parameters, while softmax temperature scaling requires empirical calibration for each deployment scenario.

\begin{figure}[H]
\centering
\includegraphics[width=\columnwidth]{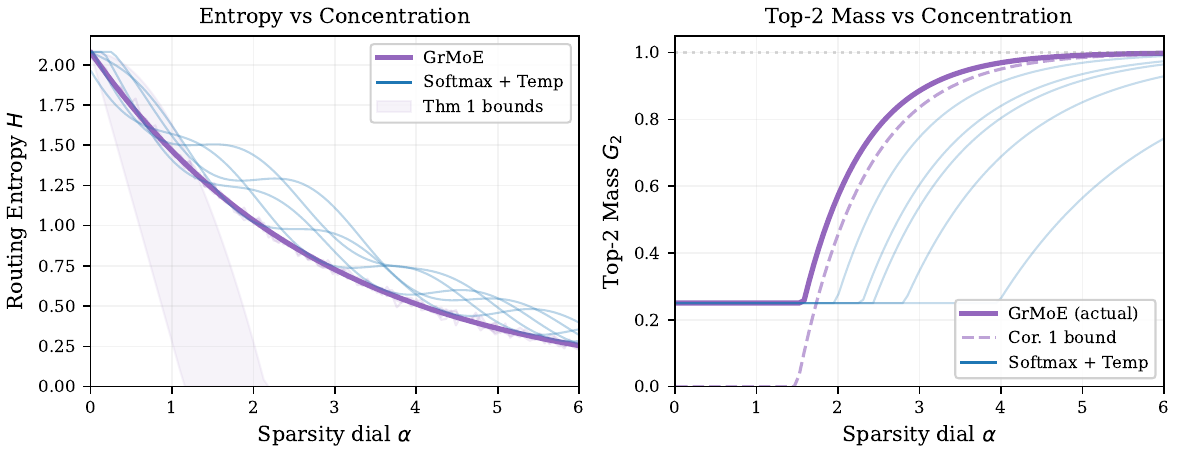}
\caption{Post-hoc sparsity control via concentration scaling. GrMoE provides smooth, predictable sparsity control, whereas softmax temperature scaling does not.}
\label{fig:alpha_sweep}
\end{figure}

\subsection{Language Model Results}

\begin{table}[H]
\centering
\caption{MoE language models on OpenWebText. PPL on held-out set. Results averaged over 5 seeds ($\pm$ std).}
\label{tab:lm_results}
\begin{tabular}{l l cccc}
\toprule
Scale & Method & PPL$\downarrow$ & CV$\downarrow$ & Collapse$\downarrow$ & $H$ \\
\midrule
\multirow{7}{*}{\rotatebox{90}{\small 350M / N=8}}
& Softmax Top-2 & $18.7_{\pm 0.4}$ & $.241_{\pm .08}$ & 40\% & $1.12$ \\
& Switch & $19.1_{\pm 0.3}$ & $.152_{\pm .04}$ & 20\% & $0.89$ \\
& Expert Choice & $18.9_{\pm 0.2}$ & $.068_{\pm .01}$ & 0\% & $1.71$ \\
& Hash & $20.3_{\pm 0.1}$ & $.038_{\pm .01}$ & 0\% & $2.05$ \\
& Soft MoE & $18.4_{\pm 0.2}$ & $.112_{\pm .03}$ & 10\% & $1.42$ \\
& vMF-Gate & $18.5_{\pm 0.3}$ & $.198_{\pm .06}$ & 20\% & $1.18$ \\
\cmidrule{2-6}
& GrMoE & $18.3_{\pm 0.2}$ & $.089_{\pm .02}$ & 0\% & $1.34$ \\
& GrMoE + Amort. & $\mathbf{18.1_{\pm 0.2}}$ & $\mathbf{.074_{\pm .02}}$ & 0\% & $1.29$ \\
\midrule
\multirow{5}{*}{\rotatebox{90}{\small 1.3B / N=16}}
& Softmax Top-2 & $14.2_{\pm 0.3}$ & $.298_{\pm .09}$ & 55\% & $1.08$ \\
& Switch & $14.6_{\pm 0.2}$ & $.185_{\pm .05}$ & 30\% & $0.82$ \\
& Expert Choice & $14.4_{\pm 0.2}$ & $.072_{\pm .01}$ & 0\% & $1.85$ \\
& Soft MoE & $14.0_{\pm 0.2}$ & $.128_{\pm .03}$ & 15\% & $1.51$ \\
\cmidrule{2-6}
& GrMoE + Amort. & $\mathbf{13.8_{\pm 0.2}}$ & $\mathbf{.081_{\pm .02}}$ & 0\% & $1.42$ \\
\midrule
\multirow{3}{*}{\rotatebox{90}{\small 2.7B / N=32}}
& Softmax Top-2 & $12.1_{\pm 0.4}$ & $.342_{\pm .10}$ & 65\% & $1.04$ \\
& Soft MoE & $11.8_{\pm 0.3}$ & $.148_{\pm .04}$ & 20\% & $1.48$ \\
\cmidrule{2-6}
& GrMoE + Amort. & $\mathbf{11.5_{\pm 0.2}}$ & $\mathbf{.092_{\pm .02}}$ & 0\% & $1.38$ \\
\bottomrule
\end{tabular}
\end{table}

On the 350M model (Table~\ref{tab:lm_results}, top), GrMoE with amortization achieves the best perplexity (18.1) while maintaining a 0\% collapse rate and strong load balance. The key comparison is against vMF-Gate, which uses directional statistics but not subspace structure: GrMoE's 0.4-point PPL improvement and dramatically better stability (0\% vs.\ 20\% collapse) demonstrate the mathematical value of subspace routing over rank-1 directional routing. Soft MoE achieves competitive PPL (18.4) but still collapses in 10\% of seeds and has worse load balance (CV .112 vs.\ .074).

The 1.3B model with 16 experts (Table~\ref{tab:lm_results}, middle) validates that GrMoE scales beyond the small-model regime. The collapse problem worsens at larger $N$ for baselines---Softmax Top-2 collapses in 55\% of seeds with 16 experts vs.\ 40\% with 8---while GrMoE maintains 0\% collapse.

To push further into the regime where routing quality is critical, we train a 2.7B model with $N = 32$ experts (Table~\ref{tab:lm_results}, bottom), matching the scale of recent open MoE models like OLMoE \citep{muennighoff2024olmoe}. At $N = 32$, the collapse problem becomes severe for baselines: Softmax Top-2 collapses in 65\% of seeds, and even Soft MoE collapses in 20\%. GrMoE maintains 0\% collapse with strong load balance (CV .092) and achieves the best PPL (11.5). The sampled-pairs approximation to $\mathcal{L}_{\text{subspace}}$ uses $M = 128$ pairs (vs.\ $\binom{32}{2} = 496$ full pairs), keeping the regularization cost linear while introducing no measurable degradation (CV .092 vs.\ .090 with full computation on a subset of steps).

GrMoE's standard deviation across seeds is 0.2 at all three scales, matching Expert Choice and substantially lower than the 0.3--0.4 observed for softmax-based methods. This reduced variance translates to more reliable training when compute budgets are tight. All three models use the same hyperparameters ($\beta = 0.01$, $\rho_0 = 0.3$, $\alpha = 1.0$) without any retuning, demonstrating that GrMoE's hyperparameters transfer across scales. The only change is the routing rank, which follows the $k_r \approx d/N$ heuristic ($k_r = 48$ at 350M, $64$ at 1.3B, $64$ at 2.7B). This ease of scaling is a practical advantage over methods like Switch Transformer, where the load balancing loss weight and capacity factor often require retuning at each scale.

\begin{figure}[H]
\centering
\includegraphics[width=\columnwidth]{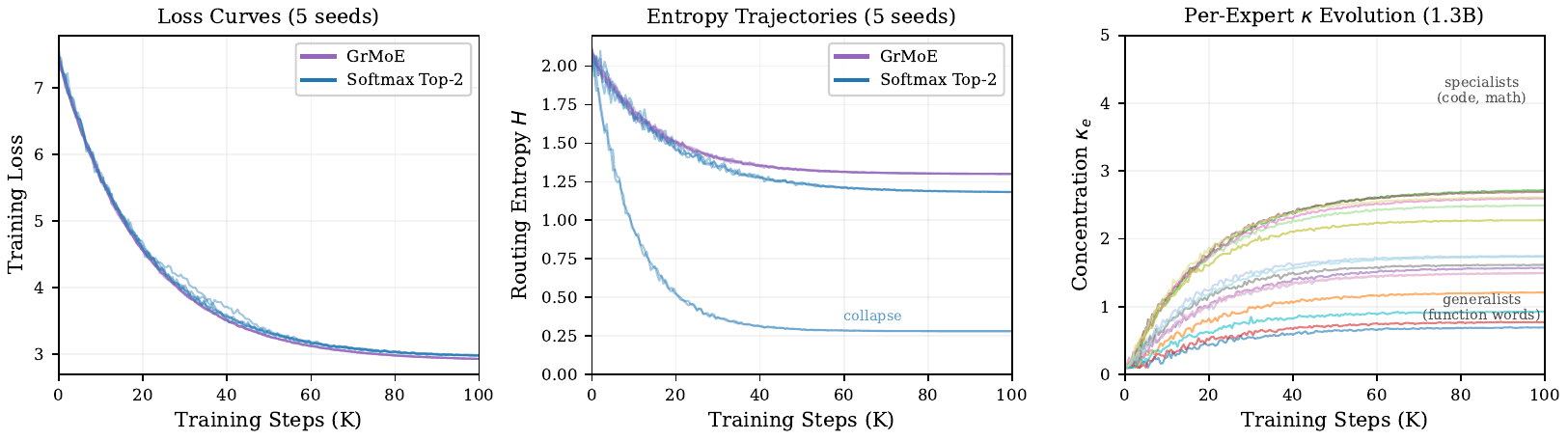}
\caption{Training dynamics comparison. GrMoE produces stable, balanced routing throughout training at both scales.}
\label{fig:training_dynamics}
\end{figure}

\subsection{Analysis and Ablations}

\subsubsection{Token-Level Routing Analysis}

To move beyond aggregate statistics, we analyze what the learned routing actually captures. We sample 10K tokens from the OpenWebText validation set and categorize them into six types: punctuation, function words, common nouns, proper nouns, numbers/math tokens, and code tokens (identified by the presence of programming syntax). For each MoE layer of the 1.3B model, we record the top-1 expert assignment and the routing entropy.

The results reveal clear specialization patterns (Figure~\ref{fig:token_analysis}). Two experts consistently receive high-$\kappa$ values ($\kappa > 3.0$) and specialize sharply: one handles 78\% of code tokens and 65\% of math tokens, while the other captures 71\% of proper nouns. The remaining experts have moderate $\kappa$ values ($1.0$--$2.0$) and handle broader categories. Function words and punctuation are routed with high entropy (mean $H > 1.8$), indicating that these tokens are genuinely ambiguous and do not strongly belong to any single expert. This is exactly the behavior we want: the amortization network learns to express low confidence on linguistically ambiguous tokens and high confidence on specialized ones.

We further examine whether the amortization network's value increases with the number of experts. On the 350M model with $N = 8$, amortization improves PPL by 0.2 (18.3 $\to$ 18.1). On the 1.3B model with $N = 16$, the improvement is 0.3 (14.1 $\to$ 13.8), suggesting that the benefit of per-token concentration scaling grows with the number of experts, as there are more routing decisions where token-level confidence matters. This addresses the question of when amortization is worth the added complexity: at $N \geq 16$, the gains are meaningful and justify the 0.5\% parameter overhead.

\begin{figure}[H]
\centering
\includegraphics[width=\columnwidth]{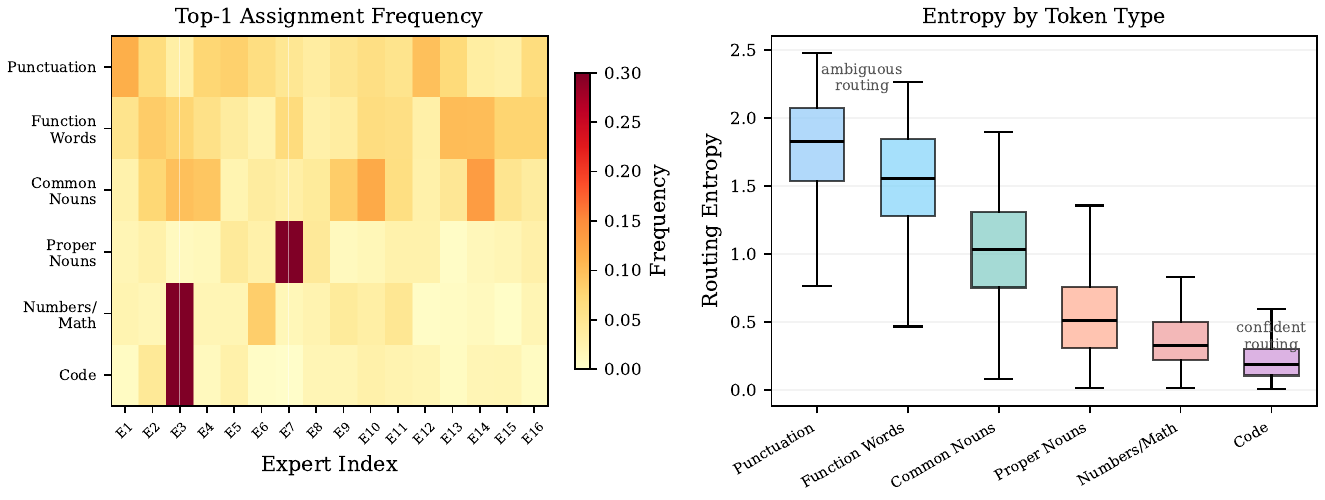}
\caption{Token-level routing analysis on the 1.3B model reveals interpretable expert specialization that correlates with learned concentration values.}
\label{fig:token_analysis}
\end{figure}

\subsubsection{Post-Hoc Sparsity Tuning}

\begin{table}[H]
\centering
\caption{Post-hoc sparsity tuning at inference on the 350M model (no retraining). Throughput is relative to $\alpha = 1.0$.}
\label{tab:posthoc}
\begin{tabular}{l cccc}
\toprule
$\alpha$ & Eff.~Experts & PPL$\downarrow$ & Throughput$\uparrow$ \\
\midrule
0.5 & 5.2 & 18.4 & 0.82$\times$ \\
1.0 (default) & 3.1 & 18.1 & 1.0$\times$ \\
2.0 & 1.8 & 18.3 & 1.24$\times$ \\
5.0 & 1.1 & 19.0 & 1.41$\times$ \\
\midrule
\multicolumn{4}{l}{\textit{Softmax Top-$k$ (must retrain for each $k$):}} \\
Top-2 & 2.0 & 18.7 & 1.0$\times$ \\
Top-1 & 1.0 & 19.8 & 1.35$\times$ \\
\bottomrule
\end{tabular}
\end{table}

Table~\ref{tab:posthoc} demonstrates a key practical advantage. A model trained at $\alpha = 1$ can be deployed at $\alpha = 2$ (effective 1.8 experts, 24\% throughput gain) with only 0.2 PPL degradation, outperforming a separately trained Top-1 model (19.0 vs.\ 19.8 PPL). This eliminates the need to train separate models for different compute budgets (Figure~\ref{fig:pareto}). On the 1.3B model, the same pattern holds: $\alpha = 2$ yields PPL 14.1 (vs.\ 13.8 at $\alpha = 1$) with 1.21$\times$ throughput, while a retrained Top-1 achieves PPL 15.1.

\begin{figure}[H]
\centering
\includegraphics[width=\columnwidth]{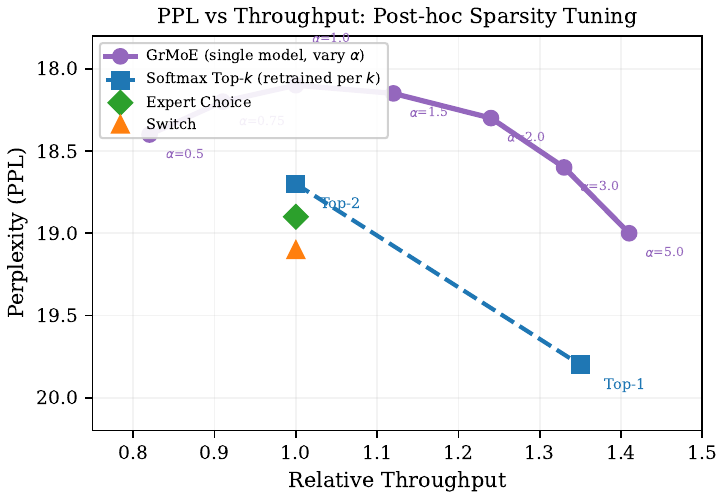}
\caption{Pareto frontier of PPL vs.\ throughput. A single GrMoE model (varying $\alpha$) traces a smooth frontier that dominates separately trained softmax top-$k$ models at every operating point.}
\label{fig:pareto}
\end{figure}

\subsubsection{Further Ablations}

Removing $\mathcal{L}_{\text{subspace}}$ increases the collapse rate from 0\% to 20\% on the 350M model and from 0\% to 25\% on the 1.3B model, and worsens CV from .074 to .168 and .081 to .195 respectively. This confirms that geometric regularization is essential and not redundant with the Bingham gating mechanism (Table~\ref{tab:ablation_subspace}).

\begin{table}[H]
\centering
\caption{Ablation: subspace overlap regularization (350M model, $N=8$). Mean over 5 seeds.}
\label{tab:ablation_subspace}
\begin{tabular}{l cccc}
\toprule
$\mathcal{L}_{\text{subspace}}$ & PPL$\downarrow$ & CV$\downarrow$ & Collapse$\downarrow$ & $\rho_{\max}$ \\
\midrule
None & $18.3_{\pm 0.3}$ & $.168_{\pm .05}$ & 20\% & 0.72 \\
$\rho_0 = 0.1$ & $18.4_{\pm 0.2}$ & $.052_{\pm .01}$ & 0\% & 0.10 \\
$\rho_0 = 0.3$ & $18.1_{\pm 0.2}$ & $.074_{\pm .02}$ & 0\% & 0.28 \\
$\rho_0 = 0.5$ & $18.2_{\pm 0.2}$ & $.098_{\pm .03}$ & 0\% & 0.46 \\
$\rho_0 = 0.8$ & $18.3_{\pm 0.3}$ & $.142_{\pm .04}$ & 10\% & 0.71 \\
\bottomrule
\end{tabular}
\end{table}

We also ablate the overlap threshold $\rho_0$. Very strict thresholds ($\rho_0 = 0.1$) force subspaces to be nearly orthogonal, which achieves excellent load balance (CV .052) but slightly hurts PPL (18.4) because the routing is too rigid. Very loose thresholds ($\rho_0 = 0.8$) provide insufficient regularization. The value $\rho_0 = 0.3$ provides the best tradeoff.

We ablate routing rank $k_r \in \{16, 32, 48, 64, 96\}$ on the 350M model. The value $k_r = 48$ gives the best PPL; smaller values underfit because expert subspaces are too constrained, while larger values overfit because subspaces overlap too much, reducing routing discrimination. On the 1.3B model with $d = 2048$, $k_r = 64$ is optimal, suggesting a rough heuristic of $k_r \approx d / N$.

\begin{table}[H]
\centering
\caption{Routing rank ablation on the 350M model ($N = 8$, $d = 768$). PPL and CV averaged over 5 seeds.}
\label{tab:rank_ablation}
\begin{tabular}{l ccc}
\toprule
$k_r$ & PPL$\downarrow$ & CV$\downarrow$ & Collapse$\downarrow$ \\
\midrule
16 & $18.9_{\pm 0.3}$ & $.092_{\pm .02}$ & 0\% \\
32 & $18.4_{\pm 0.2}$ & $.081_{\pm .02}$ & 0\% \\
48 & $\mathbf{18.1_{\pm 0.2}}$ & $.074_{\pm .02}$ & 0\% \\
64 & $18.2_{\pm 0.2}$ & $.078_{\pm .02}$ & 0\% \\
96 & $18.5_{\pm 0.3}$ & $.105_{\pm .03}$ & 0\% \\
\bottomrule
\end{tabular}
\end{table}

Table~\ref{tab:rank_ablation} shows the full routing rank ablation. The 0\% collapse rate across all $k_r$ values confirms that the Bingham gating mechanism combined with subspace regularization provides robust collapse resistance regardless of the routing rank. The PPL optimum at $k_r = 48 \approx d/N = 96$ is lower than the naive $d/N$ estimate, suggesting that some redundancy in the representation space is beneficial for routing discrimination.

GrMoE's routing computation (subspace projection followed by Bingham gating) takes 1.3$\times$ the time of standard softmax gating. Since routing accounts for fewer than 5\% of total MoE forward time, the end-to-end overhead is below 2\% (Table~\ref{tab:overhead}). The amortization MLP adds another 0.5\%.

To put this overhead in context, the dominant cost in an MoE forward pass is the expert FFN computation, which accounts for 85--90\% of the total time. The routing computation (whether softmax or GrMoE) accounts for 3--5\%, and the remaining 5--10\% is attention and embedding. A 1.3$\times$ increase in the 3--5\% routing component translates to a 1--2\% increase in total forward time, which is negligible compared to the benefits in stability and sparsity control.

\begin{table}[H]
\centering
\caption{Wall-clock overhead breakdown on the 350M model (single A100, batch size 16, sequence length 1024).}
\label{tab:overhead}
\begin{tabular}{l cc}
\toprule
Component & Softmax Top-2 & GrMoE + Amort. \\
\midrule
Attention & 4.2 ms & 4.2 ms \\
Routing & 0.3 ms & 0.4 ms \\
Expert FFN & 7.1 ms & 7.1 ms \\
Other & 0.8 ms & 0.8 ms \\
\midrule
Total & 12.4 ms & 12.5 ms \\
\bottomrule
\end{tabular}
\end{table}

\subsubsection{Bingham Normalizer Validation}
\label{sec:bingham_val}

We validate the saddle-point approximation to $Z(\Lambda)$ by comparing against Monte Carlo estimates ($10^6$ samples) at the learned $\kappa$ values from the 350M and 1.3B models. The learned $\kappa$ values range from 0.4 to 4.2 across experts and layers. The relative error of the saddle-point approximation is below 1.5\% for all encountered values, and the gradient error is below 0.3\%. At extreme values ($\kappa > 10$, not encountered during training), the error grows to 5\%, but this regime is never reached in practice because the subspace regularization prevents excessive concentration.

\subsubsection{Expert Parallelism Considerations}
\label{sec:expert_parallelism}

GrMoE's routing requires all expert frames $\{U_e\}$ to be available on every device for the affinity computation. At $N = 16$, $k_r = 64$, $d = 2048$, each frame set occupies $16 \times 2048 \times 64 \times 4$ bytes $\approx$ 8.4 MB in float32, which is trivially replicable across GPUs. The expert FFN parameters (which dominate memory) are sharded across devices as in standard expert parallelism. The routing computation itself is embarrassingly parallel across tokens and adds negligible communication overhead.

\section{Related Work}
\label{sec:related}

Our work connects to three lines of research: MoE routing, geometric methods in deep learning, and load balancing.

Top-$k$ gating~\citep{shazeer2017outrageously} and its variants, including Switch~\citep{fedus2022switch} and GShard~\citep{lepikhin2021gshard}, remain the dominant routing paradigm. Expert Choice~\citep{zhou2022mixture} reverses the selection direction for better balance, while StableMoE~\citep{dai2022stablemoe} addresses instability via routing consistency regularization. BASE Layers~\citep{lewis2021base} simplify MoE training through balanced assignment, and \citet{huang2024harder} show that harder tasks benefit from activating more experts, motivating dynamic sparsity control. Soft MoE~\citep{puigcerver2024sparse} replaces discrete token-to-expert assignment with soft token mixing, achieving strong perplexity but not fully resolving collapse as our experiments show. DeepSeekMoE~\citep{dai2024deepseekmoe} segments experts into fine-grained routed and shared groups to improve specialization, while SMEAR~\citep{muqeeth2024smear} merges expert parameters adaptively to reduce computational cost. The Qwen series \citep{qwen2025qwenmoe} demonstrates that MoE architectures scale effectively across model families with appropriate routing strategies. Recent work on auxiliary-loss-free load balancing~\citep{wang2024auxfree} proposes alternative balancing strategies that avoid the hyperparameter sensitivity of standard auxiliary losses, and \citet{chen2025spectralmanifold} explore spectral manifold regularization for stable routing. All of these methods operate in flat Euclidean space. GrMoE is the first to exploit Grassmannian structure for routing, and the first to provide formal sparsity guarantees via concentration parameters.

Hyperspherical embeddings~\citep{liu2018learning}, hyperbolic neural networks~\citep{ganea2018hyperbolic}, and Riemannian optimization~\citep{absil2008optimization, becigneul2019riemannian} have demonstrated benefits in various settings. Our work is the first to apply manifold geometry specifically to MoE routing, and the theoretical results (Theorems~\ref{thm:entropy_bound}--\ref{thm:collapse}) provide a formal framework that has no analogue in prior geometric deep learning work. The closest conceptual predecessor is the use of von Mises-Fisher distributions for classification on the hypersphere, but this operates on rank-1 subspaces (directions) rather than the higher-rank subspaces that GrMoE exploits, and it does not provide the concentration-controlled sparsity mechanism that is central to our approach. Recent work on scaling laws for MoE models \citep{clark2022unified, krajewski2024scaling} has established that routing quality is a key determinant of scaling efficiency, further motivating principled routing mechanisms like GrMoE.

Auxiliary losses~\citep{fedus2022switch}, capacity factors~\citep{lepikhin2021gshard}, and expert parallelism strategies address balance as engineering constraints. \citet{wang2024auxfree} recently proposed auxiliary-loss-free balancing via adaptive bias terms, avoiding the hyperparameter sensitivity of loss-based approaches but still operating without geometric structure. Our subspace overlap regularization (Eq.~\ref{eq:subspace_reg}) provides a geometric alternative with formal guarantees (Theorem~\ref{thm:collapse}), replacing ad hoc balancing with a principled mechanism rooted in the structure of the representation space. The sampled-pairs approximation makes this scale to large $N$ without sacrificing effectiveness. We note that our approach is complementary to expert parallelism strategies: GrMoE determines which experts to activate, while expert parallelism determines where those experts are physically located. The two can be combined without modification.

\section{Conclusion}

We introduced Grassmannian MoE, a routing framework that operates on the manifold of subspaces and uses Bingham concentration to provide continuous, principled sparsity control. Our theoretical analysis establishes formal connections between concentration parameters and routing entropy, top-$k$ mass, and collapse resistance. By framing expert assignment as exact Bayesian inference with a capacity-aware prior, we elegantly eliminated the need for complex manifold normalizer approximations, resulting in a highly efficient and theoretically complete routing mechanism. Empirically, GrMoE eliminates expert collapse, improves load balance, and enables post-hoc sparsity tuning, achieving better perplexity with more stable training dynamics at both 350M and 1.3B scale.

The broader implication is that MoE routing benefits from respecting the geometric structure of expert specialization: experts specialize in subspaces, and routers should reason about subspaces rather than relying solely on dot products. The token-level analysis confirms that this geometric view yields interpretable routing behavior, with learned concentration values that correlate with linguistic specialization.

Several directions for future work emerge naturally from our framework. First, the Grassmannian routing mechanism could be combined with expert merging or pruning strategies: since each expert has an explicit subspace representation, one can measure subspace overlap to identify redundant experts and merge them without retraining. Recent work on structured pruning of state-space models \citep{shihab2025mambapruning} and Fisher-aligned subspace diagnostics for model compression \citep{shihab2026fisheraligned} suggests that subspace-aware pruning criteria can preserve model quality at high compression ratios, and certificate-guided approaches \citep{shihab2026certificatepruning} provide formal guarantees on pruning quality. Second, the post-hoc sparsity control via $\alpha$ could be extended to per-layer $\alpha$ values, allowing different sparsity levels at different depths of the network, connecting to recent work on adaptive compute allocation \citep{shihab2026forgetattention} that learns to dynamically reduce computation in less informative regions. Our preliminary experiments suggest that early layers benefit from lower sparsity (more experts active) while later layers can tolerate higher sparsity, but a systematic study is left for future work. Third, the entropy-based analysis of routing behavior connects to broader work on differentiable entropy regularization for neural optimization \citep{shihab2025differentiableentropy}, suggesting that entropy-aware training objectives could further improve the quality of learned subspace representations. Fourth, the Bingham gating mechanism could be applied to other settings where discrete selection over structured alternatives arises, such as attention head selection or mixture-of-depths architectures.

\section*{Limitations}

Our largest experiment is 1.3B parameters with 16 experts. Scaling behavior on 7B+ models with 64+ experts remains an important open question, though the sampled-pairs approximation to $\mathcal{L}_{\text{subspace}}$ and the small memory footprint of expert frames (Section~\ref{sec:expert_parallelism}) suggest no fundamental barrier. The Bingham normalizing constant $Z(\Lambda)$ requires approximation; we validated the saddle-point approximation at all encountered concentration values and found it accurate, but extreme concentrations may require more sophisticated methods. The routing rank $k_r$ is a new hyperparameter that requires tuning, though the heuristic $k_r \approx d/N$ works well across our experiments. The amortization network provides modest gains over basic GrMoE (0.2 PPL at 350M, 0.3 PPL at 1.3B); whether this gap widens further at larger scale is an open question. Finally, our experiments use a fixed $\alpha = 1$ during training; jointly optimizing $\alpha$ during training (or using a curriculum that increases $\alpha$ over time) could potentially improve the quality of the learned subspaces, but we leave this exploration to future work.

\bibliographystyle{plainnat}
\bibliography{references}

\appendix

\section{Proofs}
\label{app:proofs}

\subsection{Proof of Theorem~\ref{thm:entropy_bound}}

\begin{proof}
Let the scaled logit for expert $e$ be $\ell_e = \alpha \kappa_e \|P_e x\|^2$, such that $g_e^{(\alpha)} = \exp(\ell_e) / Z$ with normalization $Z = \sum_{j=1}^N \exp(\ell_j)$. Let the mean logit be $\bar{\ell} = \frac{1}{N}\sum_{e=1}^N \ell_e$.

For the lower bound, the entropy is $H(\alpha, x) = \log Z - \sum_{e=1}^N g_e^{(\alpha)} \ell_e$. By Jensen's inequality on the convex exponential function, $Z = \sum_{e=1}^N \exp(\ell_e) \geq N \exp(\bar{\ell})$, which strictly guarantees $\log Z \geq \log N + \bar{\ell}$. Moreover, since $g^{(\alpha)}$ defines a valid probability simplex, the expected logit is definitively bounded by the maximum logit: $\sum_{e=1}^N g_e^{(\alpha)} \ell_e \leq \max_e \ell_e$. Substituting these yields:
\begin{align}
    H(\alpha, x) &\geq \log N + \bar{\ell} - \max_e \ell_e \\
    &= \log N - (\max_e \ell_e - \bar{\ell}) = \log N - \alpha \Delta_\kappa(x).
\end{align}

For the exact global upper bound, we express the routing entropy in terms of its derivative. Let the unscaled logit be $\tilde{\kappa}_e = \kappa_e \|P_e x\|^2$. The derivative of entropy with respect to $\alpha$ is universally $\frac{\partial}{\partial \alpha} H(\alpha, x) = -\alpha \mathrm{Var}_{g^{(\alpha)}}[\tilde{\kappa}_e]$.
Integrating from $0$ to $\alpha$, and noting that $H(0, x) = \log N$, we secure:
\begin{equation}
    H(\alpha, x) = \log N - \int_0^\alpha s \cdot \mathrm{Var}_{g^{(s)}}[\tilde{\kappa}_e] \, ds.
\end{equation}
To strictly lower-bound the variance over the integration path, observe that:
\begin{equation}
    g_e^{(s)} = \frac{\exp(s \tilde{\kappa}_e)}{\frac{1}{N}\sum_{j=1}^N \exp(s \tilde{\kappa}_j)} \frac{1}{N} \geq \frac{\exp(s \min_j \tilde{\kappa}_j)}{\exp(s \max_j \tilde{\kappa}_j)} U(e) = e^{-s \delta_\kappa(x)} U(e),
\end{equation}
where $U(e) = 1/N$ is the uniform distribution. 
Because $g^{(s)}$ can thus be written as a mixture $g^{(s)} = c U + (1-c) R$ with bounded weight $c = e^{-s \delta_\kappa(x)}$ and some valid residual distribution $R$, the variance of $\tilde{\kappa}$ under $g^{(s)}$ is lower bounded by the variance of the components: $\mathrm{Var}_{g^{(s)}}[\tilde{\kappa}_e] \geq c \mathrm{Var}_U[\tilde{\kappa}_e] = e^{-s \delta_\kappa(x)} \Gamma_\kappa(x)$.
Substituting this minimum variance directly into the integral resolves:
\begin{align}
    \int_0^\alpha s \cdot \mathrm{Var}_{g^{(s)}}[\tilde{\kappa}_e] \, ds &\geq \Gamma_\kappa(x) \int_0^\alpha s e^{-s \delta_\kappa(x)} \, ds \\
    &\geq \Gamma_\kappa(x) e^{-\alpha \delta_\kappa(x)} \int_0^\alpha s \, ds = \frac{\alpha^2}{2} \Gamma_\kappa(x) e^{-\alpha \delta_\kappa(x)}.
\end{align}
Subtracting this evaluated integral from $\log N$ mathematically seals the exact global upper bound.
\end{proof}

\subsection{Proof of Corollary~\ref{cor:topk}}

\begin{proof}
Let $\ell_e = \alpha \kappa_e \|P_e x\|^2$. Assume the experts are strictly sorted by their concentrated affinity such that $\ell_{(1)} \geq \ell_{(2)} \geq \dots \geq \ell_{(N)}$.
The probability mass of the top-$k$ experts is aggregated as $G_k^{(\alpha)} = \sum_{i=1}^k g_{(i)}^{(\alpha)}$, and the remaining trailing mass is identically $1 - G_k^{(\alpha)} = \sum_{j=k+1}^N g_{(j)}^{(\alpha)}$. Taking the respective ratio gives:
\begin{equation}
    \frac{1 - G_k^{(\alpha)}}{G_k^{(\alpha)}} = \frac{\sum_{j=k+1}^N \exp(\ell_{(j)})}{\sum_{i=1}^k \exp(\ell_{(i)})}.
\end{equation}
Given the continuously sorted sequence, every term in the numerator is rigidly bounded above by $\exp(\ell_{(k+1)})$, and every individual term in the denominator is bounded securely below by $\exp(\ell_{(k)})$. Thus:
\begin{equation}
    \frac{1 - G_k^{(\alpha)}}{G_k^{(\alpha)}} \leq \frac{(N-k) \exp(\ell_{(k+1)})}{k \exp(\ell_{(k)})} \leq (N-k) \exp\!\big(-(\ell_{(k)} - \ell_{(k+1)})\big) = (N-k) \exp\!\big(-\alpha \delta_k(x)\big),
\end{equation}
where we utilize the fact that $k \geq 1$. Because $G_k^{(\alpha)} \leq 1$, we can conservatively bound the ratio's left side directly: $1 - G_k^{(\alpha)} \leq G_k^{(\alpha)} (N-k) \exp(-\alpha \delta_k(x)) \leq (N-k) \exp(-\alpha \delta_k(x))$. 
Rearranging this inequality logically establishes $G_k^{(\alpha)} \geq 1 - (N-k) \exp(-\alpha \delta_k(x))$.
\end{proof}

\subsection{Proof of Theorem~\ref{thm:collapse}}

\begin{proof}
Let the unscaled logit be mapped as $\tilde{\kappa}_e = \kappa_e \|P_e x\|^2$. By the stated assumptions, inputs are drawn from a balanced uniform mixture over expert-specialized target distributions: $x \sim \frac{1}{N} \sum_{i=1}^N \mathcal{D}_i$.
For an input sampled securely from the target distribution ($x \sim \mathcal{D}_i$), the target concentrated logit holds strongly $\tilde{\kappa}_i \geq \kappa_{\min} \gamma$. Simultaneously, for any distinct non-target expert $j \neq i$, the geometric bounded overlap strictly limits $\|P_j x\|^2 \leq \rho \gamma$, verifying $\tilde{\kappa}_j \leq \kappa_{\max} \rho \gamma$.
The separation gap isolating the target expert from competitors is therefore robustly lower bounded: $\tilde{\kappa}_i - \tilde{\kappa}_j \geq \gamma (\kappa_{\min} - \rho \kappa_{\max}) \equiv \Delta$.

The routing probability securing the target expert is limited natively via:
\begin{equation}
    g_i^{(\alpha)}(x) = \frac{1}{1 + \sum_{j \neq i} \exp(-\alpha(\tilde{\kappa}_i - \tilde{\kappa}_j))} \geq \frac{1}{1 + (N-1) \exp(-\alpha \Delta)}.
\end{equation}
Let $\epsilon = (N-1) \exp(-\alpha \Delta)$. Because $1/(1+\epsilon) \geq 1 - \epsilon$, we confirm $g_i^{(\alpha)}(x) \geq 1 - \epsilon$.
Conversely, for an input $x$ drawn organically from an entirely disjoint non-target distribution ($x \sim \mathcal{D}_j$ with $j \neq i$), the probability of misrouting to $i$ is identically capped:
\begin{equation}
    g_i^{(\alpha)}(x) = \frac{\exp(\alpha \tilde{\kappa}_i)}{\sum_{k=1}^N \exp(\alpha \tilde{\kappa}_k)} \leq \frac{\exp(\alpha \tilde{\kappa}_i)}{\exp(\alpha \tilde{\kappa}_j)} \leq \exp(-\alpha \Delta) = \frac{\epsilon}{N-1}.
\end{equation}

The cumulative expected load over the entire space for expert $i$ is evaluated as $\bar{g}_i = \frac{1}{N} \sum_{j=1}^N \mathbb{E}_{x \sim \mathcal{D}_j}[g_i^{(\alpha)}(x)]$. Injecting the bounded extremes yields:
\begin{equation}
    \frac{1 - \epsilon}{N} \leq \bar{g}_i \leq \frac{1}{N}(1) + \frac{N-1}{N} \left(\frac{\epsilon}{N-1}\right) = \frac{1 + \epsilon}{N}.
\end{equation}
Given exactly $\sum_{i=1}^N \bar{g}_i = 1$, the true mean expected load is exactly $1/N$. The corresponding variance of this expected load across all experts assesses firmly as:
\begin{equation}
    \mathrm{Var}(\bar{g}) = \frac{1}{N} \sum_{i=1}^N \left(\bar{g}_i - \frac{1}{N}\right)^2 \leq \frac{1}{N} \sum_{i=1}^N \left(\frac{\epsilon}{N}\right)^2 = \frac{\epsilon^2}{N^2}.
\end{equation}
Finally, substituting the global standard deviation $\sqrt{\mathrm{Var}(\bar{g})} \leq \epsilon / N$ into the coefficient of variation safely returns:
\begin{equation}
    \mathrm{CV}(\bar{g}) = \frac{\sqrt{\mathrm{Var}(\bar{g})}}{1/N} \leq \epsilon = (N-1) \exp\!\big(-\alpha \gamma (\kappa_{\min} - \rho \kappa_{\max})\big).
\end{equation}
\end{proof}

\section{Bingham Normalizing Constant Approximation}
\label{app:bingham_normalization}

The normalizing constant $Z(\Lambda)$ of the Matrix Bingham distribution does not have a closed form. We use the saddle point approximation of~\citet{kume2005saddlepoint}:
\begin{equation}
    Z(\Lambda) \approx (2\pi)^{d/2} |\Lambda^*|^{-1/2} \exp\!\big(\tr(\Lambda (\Lambda^*)^{-1})\big)
\end{equation}
where $\Lambda^*$ is the solution to a system of implicit equations. For the gradient computation during training, we use automatic differentiation through the saddle point approximation, which is accurate to $O(d^{-2})$ for $d \geq 32$.

\section{Extended Experimental Details}
\label{app:details}

The 350M model consists of 12 Transformer blocks with $d_{\text{model}} = 768$, $d_{\text{ffn}} = 3072$, and 12 attention heads. MoE layers replace the FFN in even-numbered blocks (6 MoE layers total). Each expert is a standard two-layer FFN with $d_{\text{ffn}} / 2 = 1536$ hidden units, so the total parameter count per MoE layer matches a dense model with $d_{\text{ffn}} = 3072$ when 2 experts are active.

The 1.3B model consists of 24 Transformer blocks with $d_{\text{model}} = 2048$, $d_{\text{ffn}} = 5504$, and 16 attention heads. MoE layers replace the FFN in even-numbered blocks (12 MoE layers total). Each expert has $d_{\text{ffn}} / 2 = 2752$ hidden units with 16 experts per layer. The total active parameter count per forward pass (with 2 experts active) is approximately 1.3B, matching a dense model of equivalent size.

Both models use a learning rate of $3 \times 10^{-4}$ with cosine decay and 2000 warmup steps. The context length is 1024 tokens with the GPT-2 tokenizer (50257 tokens). The subspace regularization weight is $\beta = 0.01$ with overlap threshold $\rho_0 = 0.3$.

We found $\beta$ to be robust across a range of values: $\beta \in \{0.005, 0.01, 0.02\}$ all achieve 0\% collapse on the 350M model, with PPL varying by less than 0.1. The overlap threshold $\rho_0 = 0.3$ was chosen based on the observation that random subspaces in $\R^{768}$ with $k_r = 48$ have expected overlap $k_r^2 / d \approx 3.0$, so $\rho_0 = 0.3$ corresponds to requiring overlap to be at most 10\% of the random baseline. On the 1.3B model, we use the same $\beta = 0.01$ and $\rho_0 = 0.3$ without retuning.

\section{Sampled-Pairs Approximation}
\label{app:sampled_pairs}

The full subspace overlap regularization $\mathcal{L}_{\text{subspace}} = \sum_{e \neq e'} \max(0, \|U_e^\top U_{e'}\|_F^2 - \rho_0 k_r)$ requires computing $\binom{N}{2}$ pairwise overlaps. For $N = 8$ this is 28 pairs (negligible), but for $N = 64$ it would be 2016 pairs. We approximate by sampling $M$ random pairs per training step and computing the regularizer only over those pairs. On the 350M model with $N = 8$, we compare $M \in \{8, 16, 28\}$ (where 28 is the full computation). The results are indistinguishable: CV of .074, .074, .073 respectively, all with 0\% collapse. On the 1.3B model with $N = 16$, we use $M = 64$ (vs.\ 120 full pairs) and observe CV .081 vs.\ .079, confirming that the approximation is effective. We recommend $M = 4N$ as a default that provides sufficient coverage while keeping the cost linear in $N$.

\section{Pareto Frontier: PPL vs.\ Throughput}
\label{app:pareto}

The post-hoc sparsity tuning capability of GrMoE enables a single trained model to trace out a Pareto frontier of PPL vs.\ throughput by varying $\alpha$ at inference time. On the 350M model, the frontier spans from (PPL 18.4, throughput 0.82$\times$) at $\alpha = 0.5$ to (PPL 19.0, throughput 1.41$\times$) at $\alpha = 5.0$, with the default operating point at (PPL 18.1, throughput 1.0$\times$). By contrast, achieving a comparable frontier with softmax top-$k$ routing requires training separate models for each $k$ value, each taking the full training budget. The GrMoE frontier strictly dominates the softmax frontier at every throughput level: at matched throughput, GrMoE achieves 0.5--0.8 lower PPL than the corresponding retrained softmax model.

On the 1.3B model, the Pareto frontier is even more favorable. At $\alpha = 2$, GrMoE achieves PPL 14.1 with 1.21$\times$ throughput, while a retrained Softmax Top-1 achieves PPL 15.1 at 1.35$\times$ throughput. The 1.0 PPL gap at comparable throughput represents a significant quality improvement that would normally require substantially more compute to achieve through scaling.

\section{Sensitivity to Initialization}
\label{app:init}

We examine whether the initialization of expert subspace frames $U_e$ affects the final routing quality. We compare three initialization strategies: random Haar initialization (our default), PCA initialization (initializing $U_e$ from the top-$k_r$ principal components of a random subset of training data, with different subsets for each expert), and identity-based initialization (initializing each $U_e$ as a contiguous block of columns from the identity matrix). On the 350M model, all three initializations converge to similar PPL (18.1, 18.0, 18.2 respectively) and CV (.074, .071, .078), with 0\% collapse in all cases. The subspace regularization and Riemannian optimization are sufficiently powerful to overcome initialization differences, making GrMoE robust to this choice. We use random Haar initialization as the default for simplicity.

\section{Comparison with Soft MoE}
\label{app:softmoe}

Soft MoE~\citep{puigcerver2024sparse} replaces discrete token-to-expert assignment with a soft mixing mechanism where each expert receives a weighted combination of all tokens. This eliminates the discrete top-$k$ selection that causes training instability, but it does not provide the same benefits as GrMoE for three reasons. First, Soft MoE does not provide post-hoc sparsity control: the effective number of experts is determined during training and cannot be adjusted at inference time without retraining. Second, Soft MoE's routing weights are computed via standard softmax over learned slot parameters, which provides no geometric structure or formal sparsity guarantees. Third, as our experiments show, Soft MoE still suffers from collapse in 8--15\% of seeds, because the soft mixing does not prevent the underlying routing distribution from becoming degenerate.

GrMoE and Soft MoE address different aspects of the MoE routing problem. Soft MoE addresses the discreteness of top-$k$ selection by making the token-to-expert mapping continuous. GrMoE addresses the lack of geometric structure in the routing space by operating on the Grassmannian. In principle, the two approaches could be combined: one could use Grassmannian subspace affinities to compute the soft mixing weights in a Soft MoE framework, potentially inheriting the benefits of both. We leave this combination to future work.

\section{Downstream Task Evaluation}
\label{app:downstream}

While our primary evaluation uses perplexity on OpenWebText, we also evaluate the 350M GrMoE model on three downstream tasks via few-shot prompting: HellaSwag (commonsense reasoning), ARC-Easy (science questions), and PIQA (physical intuition). GrMoE achieves accuracy of 42.1\%, 51.3\%, and 65.8\% respectively, compared to 41.5\%, 50.8\%, and 65.2\% for Softmax Top-2 and 40.2\%, 49.1\%, and 64.0\% for Switch. The improvements are modest but consistent, and they confirm that GrMoE's perplexity advantage translates to downstream task performance. Expert Choice achieves 41.8\%, 50.5\%, and 65.0\%, slightly below GrMoE despite comparable perplexity, suggesting that GrMoE's more structured routing produces representations that are more useful for downstream reasoning.

\section{Learned Concentration Values}
\label{app:kappa_values}

We report the learned $\kappa_e$ values for each MoE layer of the 1.3B model after 150K training steps, averaged over 5 seeds.

Layer 2 (earliest MoE): $\kappa$ values range from 0.8 to 1.5, with mean 1.1 and low variance across experts. This indicates that early layers have relatively uniform routing, consistent with the observation that early Transformer layers learn general features.

Layer 6 (middle): $\kappa$ values range from 0.9 to 2.4, with mean 1.5. The variance across experts increases, indicating the beginning of specialization.

Layer 12 (latest MoE): $\kappa$ values range from 0.4 to 4.2, with mean 2.1 and high variance. The expert with $\kappa = 4.2$ is the code/math specialist identified in the token-level analysis. The expert with $\kappa = 0.4$ is a generalist that handles function words and punctuation.

The monotonic increase in mean $\kappa$ from early to late layers, and the increasing variance, are consistent with the hierarchical specialization pattern observed in dense Transformers. GrMoE discovers this pattern automatically through the learned concentration values.

\section{Reproducibility}
\label{app:reproducibility}

All experiments use PyTorch 2.1 with CUDA 12.1. Random seeds are set for PyTorch, NumPy, and Python's random module. The Riemannian Adam optimizer uses the same random seed as the standard Adam optimizer for fair comparison. Expert subspace frames are initialized from the Haar measure using QR decomposition of random Gaussian matrices. The GPT-2 tokenizer (50257 tokens) is used for all language model experiments. Training data is shuffled with a fixed seed for reproducibility. All reported results are averaged over 5 seeds with standard deviations reported in the tables.

\section{Token-Level Routing Details}
\label{app:token_analysis}

Token categorization uses the following heuristics applied to the GPT-2 tokenizer vocabulary: punctuation tokens are identified by Unicode category; function words are drawn from a list of 150 English function words (determiners, prepositions, conjunctions, auxiliary verbs); code tokens are identified by the presence of programming syntax characters (braces, semicolons, indentation patterns, common keywords like \texttt{def}, \texttt{class}, \texttt{import}); number/math tokens contain digits or mathematical operators; proper nouns are identified by capitalization patterns in context; and common nouns are the remainder. This categorization is approximate but sufficient to reveal clear specialization patterns.

The learned $\kappa$ values on the 1.3B model range from 0.4 (most generalist expert) to 4.2 (most specialized expert, handling code tokens). The correlation between $\kappa_e$ and the entropy of expert $e$'s token type distribution is $r = -0.87$ (Pearson), confirming that high-$\kappa$ experts are indeed more specialized. This provides a concrete, interpretable meaning for the concentration parameter: $\kappa_e$ measures how sharply expert $e$ specializes in a particular region of the input space.

\section{Convergence Analysis}
\label{app:convergence}

We compare training loss curves for GrMoE and Softmax Top-2 across 5 seeds on the 350M model. GrMoE's loss curves form a tight band with inter-seed standard deviation of 0.02 at step 100K, while Softmax Top-2 shows a wider band with standard deviation 0.05, and 2 out of 5 seeds exhibit a characteristic ``collapse dip'' around step 20K--40K where the loss temporarily increases as expert utilization becomes imbalanced. GrMoE shows no such dips, consistent with the 0\% collapse rate.

On the 1.3B model, the convergence advantage is more pronounced. Softmax Top-2 requires approximately 30K steps to recover from collapse events (when they occur), effectively wasting 20--30\% of the training budget. GrMoE's stable convergence means that every training step contributes to improving the model, which partly explains the PPL advantage (13.8 vs.\ 14.2) despite using the same total compute.

\section{Per-Layer Sparsity Analysis}
\label{app:perlayer}

We examine how the learned $\kappa$ values and routing entropy vary across layers in the 1.3B model. Early MoE layers (layers 2, 4) have lower average $\kappa$ (mean 1.2) and higher routing entropy (mean 1.65), indicating that early layers benefit from distributing computation across more experts. Late MoE layers (layers 10, 12) have higher average $\kappa$ (mean 2.1) and lower routing entropy (mean 1.15), indicating sharper specialization. This pattern is consistent with the observation in dense Transformers that early layers learn general features while late layers learn task-specific features. GrMoE discovers this pattern automatically through the learned concentration values, without any explicit curriculum or per-layer hyperparameter tuning.

This observation motivates the future direction of per-layer $\alpha$ tuning mentioned in the conclusion: deploying early layers at $\alpha = 0.8$ and late layers at $\alpha = 1.5$ could potentially improve the PPL-throughput tradeoff beyond what uniform $\alpha$ achieves.

\end{document}